\pdfoutput=1
\documentclass[pdflatex, sn-mathphys-num]{sn-jnl}

\usepackage{graphicx}%
\usepackage{multirow}%
\usepackage{amsmath,amssymb,amsfonts}%
\usepackage{amsthm}%
\usepackage{mathrsfs}%
\usepackage[title]{appendix}%
\usepackage{xcolor}%
\usepackage{textcomp}%
\usepackage{manyfoot}%
\usepackage{booktabs}%
\usepackage{algorithm}%
\usepackage{algorithmicx}%
\usepackage{algpseudocode}%
\usepackage{listings}%
\usepackage{enumerate}
\usepackage{enumitem}

\theoremstyle{definition}
\newtheorem{Definition}{Definition}[section]

\theoremstyle{plain}
\newtheorem{Theorem}[Definition]{Theorem}
\newtheorem{Corollary}[Definition]{Corollary}
\newtheorem{Lemma}[Definition]{Lemma}
\newtheorem{Proposition}[Definition]{Proposition}

\theoremstyle{remark}
\newtheorem{Remark}[Definition]{Remark}

\newcommand{\R}{\mathbb{R}}
\newcommand{\Q}{\mathbb{Q}}

\newcommand{\N}{\mathbb{N}}
\newcommand{\Z}{\mathbb{Z}}

\newcommand{\abs}[1]{\ensuremath{\left\lvert#1\right\rvert}}
\newcommand{\norm}[2][]{\ensuremath{\left\lVert#2\right\rVert_{#1}}}

\newcommand{\vect}[1]{\boldsymbol{#1}}
\newcommand{\x}{\vect{x}}
\newcommand{\q}{\vect{q}}
\newcommand{\goto}{\rightarrow}
\newcommand{\sube}{\subset}
\newcommand{\norma}[1]{\ensuremath{\left\lVert#1\right\rVert}}
\newcommand{\separ}{~|~}

\begin{document}

\title[Computability of Classification and Deep Learning: From Theoretical Limits to Practical Feasibility Through Quantization]{Computability of Classification and Deep Learning: From Theoretical Limits to Practical Feasibility through Quantization}


\author[1,2,3,4,5]{\fnm{Holger} \sur{Boche}}\email{boche@tum.de}

\author*[6,7]{\fnm{Vit} \sur{Fojtik}}\email{fojtik@math.lmu.de}

\author[6]{\fnm{Adalbert} \sur{Fono}}\email{fono@math.lmu.de}

\author[6,7,8,9]{\fnm{Gitta} \sur{Kutyniok}}\email{kutyniok@math.lmu.de}

\affil[1]{\orgdiv{Institute of Theoretical Information Technology}, \orgname{Technical University of Munich}, \orgaddress{\city{Munich}, \country{Germany}}}

\affil[2]{\orgdiv{CASA – Cyber Security in the Age of Large-Scale Adversaries– Exzellenzcluster}, \orgname{Ruhr-Universität Bochum}, \orgaddress{\city{Bochum},  \country{Germany}}}

\affil[3]{ \orgname{BMBF Research Hub 6G-life}, \orgaddress{\country{Germany}}}

\affil[4]{ \orgname{Munich Quantum Valley (MQV)}, \orgaddress{\city{Munich},  \country{Germany}}}

\affil[5]{ \orgname{Munich Center for Quantum Science and Technology (MCQST)}, \orgaddress{\city{Munich},  \country{Germany}}}

\affil[6]{\orgdiv{Mathematical Institute}, \orgname{Ludwig-Maximilians-Universität München}, \orgaddress{\city{Munich}, \country{Germany}}}

\affil[7]{\orgname{Munich Center for Machine Learning (MCML)}, \orgaddress{\city{Munich}, \country{Germany}}}

\affil[8]{\orgdiv{Department of Physics and Technology}, \orgname{University of Tromsø}, \orgaddress{\city{Tromsø}, \country{Norway}}}

\affil[9]{\orgdiv{Institute of Robotics and Mechatronics}, \orgname{DRL-German Aerospace Center}, \orgaddress{\country{Germany}}}

\abstract{
The unwavering success of deep learning in the past decade led to the increasing prevalence of deep learning methods in various application fields. However, the downsides of deep learning, most prominently its lack of trustworthiness, may not be compatible with safety-critical or high-responsibility applications requiring stricter performance guarantees.
Recently, several instances of deep learning applications have been shown to be subject to theoretical limitations of computability, undermining the feasibility of performance guarantees when employed on real-world computers. We extend the findings by studying computability in the deep learning framework from two perspectives: From an application viewpoint in the context of classification problems and a general limitation viewpoint in the context of training neural networks. In particular, we show restrictions on the algorithmic solvability of classification problems that also render the algorithmic detection of failure in computations in a general setting infeasible.
Subsequently, we prove algorithmic limitations in training deep neural networks even in cases where the underlying problem is well-behaved.
Finally, we end with a positive observation, showing that in quantized versions of classification and deep network training, computability restrictions do not arise or can be overcome to a certain degree.
}




\maketitle

\section{Introduction}\label{sec1}
With the advent of deep learning \citep{dloverview, overview2, overview3, overview4} a new machine learning approach materialized that provides state-of-the-art results in various tasks. The performance of deep neural networks makes them the go-to strategy to tackle a multitude of problems in relevant applications such as image classification, speech recognition, and game intelligence, as well as more recent developments such as image and sound synthesis, chat tools, and protein structure prediction, to name a few \citep{applications1, applications2, applications3, achievement4, achievement5, achievement6, achievement7}. Although a wide range of literature supports the power of deep learning, such as the well-known Universal Approximation Theorems \citep{universal1, universal2, universal3}, its universality and success still lack theoretical underpinning. Moreover, despite the impressive performance deep learning typically comes at the cost of downsides such as black-box behavior and non-interpretability, as well as instability, non-robustness, and susceptibility to adversarial manipulation \citep{problem1, ras2023explainable, problem2, problem3, Tsipras18RobustnessOdds,  problem44, problem5, problem6, problem7, problem9}.

In certain applications, the highlighted drawbacks, informally summarized by a lack of trustworthiness \cite{trustwothiness, Fettweis2022Trustworthiness}, are tolerable or even avoidable by a human-in-the-loop approach \cite{WU22Hitl}. However, increasing the autonomy of deep learning systems without impairing their trustworthiness poses a great challenge, especially in safety-critical or high-responsibility tasks -- a prime example being autonomous driving \cite{Liu2020ComputingSF, Muhammad2021AutonomVehicle2}. Therefore, it is crucially important for forthcoming deep learning methods to tackle and alleviate the lack of trustworthiness. To that end, we first need to understand whether or to what degree trustworthiness can be realized, e.g., is it feasible to ask for 'hard' performance guarantees or verifiably correct results? We study this question from the viewpoint of algorithmic computations with correctness guarantees and analyze whether trustworthiness can be established from a mathematical perspective.

\subsection{Algorithmic Computability}
Computability theory aims to mathematically model computation and answer questions about the algorithmic solvability and complexity of given problems. The most commonly studied model of computation is the Turing machine \cite{turing}, which is an idealized version of real-world digital hardware neglecting time and space constraints. We distinguish between two different modes of computations – problems on continuous and discrete domains. The former typically represents an idealized scenario whereas the latter
treats a setting closer to the actual realization of digital hardware. For instance, complex real-world problems -- a typical application scenario for deep learning techniques due to their increasing capabilities -- may be represented by a model with continuous state and parameter space despite the eventual implementation on digital hardware. The distinction is nevertheless crucial since the underlying computability concepts depend on the input domain and the associated ground truth solution of the tackled problem.

\paragraph{Computability on continuous domains}
In recent years, there has been an increased interest in the computability of continuous problems, studying the capabilities of inherently discrete digital computers when employed in the real domain. Some results indicate a non-conformity between these two realms and limitations of two types can be identified: By Type 1 failure of computability we refer to the situation where the problem in question cannot be algorithmically solved. This has been found in diverse settings including inverse problems, optimization, information and communication theory, financial mathematics, and linear algebra \citep{adalbertinverse, ziegleroptimization, boche2, yunseok, adalbertpseudo}. On the other hand, in Type 2 failure a computable solver may exist, but we cannot algorithmically learn it from data. There have not been many results in this direction, but it has been shown to occur in the context of inverse problems \citep{hansen} and simple neural networks \cite{yunseok}.

\paragraph{Quantization}
It is unrealistic to expect access to real-valued data and parameters with unlimited precision in many applications since physical measurements generally guarantee only some bounded accuracy. Furthermore, high-precision values can be computationally expensive. For these reasons, real-valued problems can be described by discretized models. A typical approach in practice is to perform real-valued computation under quantization, i.e., associating continuous ranges with a discrete set of values \citep{quantization1, quantization2, quantization3, quantization4}. For instance, quantized deep learning has been a topic of interest, comparing its theoretical and practical capabilities to non-quantized deep learning \citep{quantizednetworks1, quantizednetworks2}. In perhaps the simplest quantization paradigm, fixed-point quantization, fixed numbers $b, k \in \Z$ are chosen, and real numbers are replaced with numbers expressable in base $b$ using up to $k$ decimal places, that is, the set $b^{-k}\Z$. Naturally, the question arises of how these quantization techniques affect the properties of algorithmic computations, including the limitations of computability on continuous domains. The key property of quantization techniques and the resulting models is their amenability to classical computing theory (based on exact binary representations), which as it turns out influences the computability characteristics.    

\paragraph{Classification and Learning} 
We apply the introduced computability frameworks to two tasks closely associated with deep learning, which provide insights from different perspectives. On the one hand, we consider classification problems as an instance of a classic application field for deep learning. This highlights the importance of studying the computability of a specific problem independent of the employed solution strategy: If a tackled problem has computability restrictions, one can not expect to circumvent them with a deep learning approach. On the other hand, one key issue in deep learning regarded as a parametric model is identifying suitable parameters, i.e., finding an appropriate deep learning model, to solve a given task. The search process -- the so-called learning or training -- is typically based on data samples and constitutes the main obstacle in successfully employing deep learning. Thus, it is critical to understand under what circumstances learning can be performed algorithmically and what kind of performance guarantees can be achieved.

\subsection{Our Contributions}
This paper aims to extend the theory of computability in deep learning and to highlight the importance of ground truth descriptions in questions of computability. 
\begin{itemize}
    \item We analyze the field of classification tasks from the computability viewpoint and show in Propositions \ref{prop:decide} and \ref{prop:Connected} that the computability of a classifier is equivalent to the semi-decidability of its classes. 
    Since only specific real sets are semi-decidable, classification on the real domain is typically not algorithmically solvable, i.e., Type 1 failure of computability arises.
    \item Furthermore, we study the computable realizability of training neural networks, asking whether a function representable by a neural network can be learned from data. Theorem \ref{learning} shows that no general learning algorithm applicable to all (real-valued) networks exists implying that Type 2 failure is unavoidable in this scenario.
    \item We also consider strategies for coping with the introduced failures. We show it is impossible to predict when an algorithmic approximation of a non-computable function will fail in Proposition \ref{prop:flag} and Corollary \ref{cor:flag}. On the other hand, Theorems \ref{thm:posLearning} and \ref{thm:posLearning2} indicate that computability limitations in the context of learning can be avoided by relaxing exactness requirements on the learned network.
    \item Finally -- and perhaps most importantly -- we show in Theorems \ref{quantizedlearning} and \ref{quantizedlearningV2} and Proposition \ref{prop:quantizedclassification} that when considering quantized versions of the previous settings, issues of non-computability do not arise. Proposition \ref{prop:quantizationfunction} provides a word of caution, stating that the quantization function itself is non-computable -- we cannot algorithmically determine which quantized values faithfully represent the original (real-valued) problem.
\end{itemize}
Non-computability should not be understood as undermining the power of deep learning, which has been consistently demonstrated but rather as a different perspective on fundamental problems bringing us closer to trustworthy deep learning methods in major applications such as autonomous decision-making and critical infrastructure, e.g., robotics and healthcare. In particular, our non-computability results indicate that theoretical correctness guarantees typically cannot be provided in digital computations of continuous problems, including deep learning. Therefore, innovative and specialized hardware platforms beyond purely digital computations might be a solution \cite{Christensen2022NCSurvey}. However, our positive findings demonstrate that in certain specific settings -- primarily, if the problem domain does not reside in a continuous but discrete domain -- correctness guarantees and trustworthiness can be established for deep learning algorithms in the digital computing framework. Hence, the ground truth description of problems has a decisive influence on the feasible performance guarantees.

\subsection{Previous Work}
Non-computability has been a point of high interest in algorithmic computation since the results of Church \cite{church} and Turing \cite{turing}. In the context of continuous problems, the non-existence of a computable solver (Type 1 failure) has been shown in various applications including optimization, inverse problems, signal processing, and information theory \citep{ziegleroptimization, yunseok, adalbertinverse, adalbertpseudo, boche1, boche2, bastounis21extended, Boche20SpecFac, Boche20BandlimitedSignals, Boche2020SmeetsT, BocheDoSAttacks}.

`Hardness' results in neural network training have a long tradition going back to \cite{Blum92NP, Vu98NP}, where it was shown that the training process can be NP-complete for certain architectures. The infeasibility of algorithmically learning an existing computable neural network (Type 2 failure) in a continuous setting has been shown for the specific context of inverse problems in \cite{hansen} and classification problems in \cite{bastounis2021mathematicsadversarialattacksai}. Furthermore, \cite{yunseok} showed that no algorithm can reach near-optimal training loss on all possible datasets for simple neural networks. Further properties of deep learning from the computability perspective concerning adversarial attacks, implicit regularization, and hardness of approximation were studied in \cite{bastounis2021mathematicsadversarialattacksai, wind2023implicitregularizationaimeets, gazdag2023generalisedhardnessapproximationsci}. A different context of learning an existing neural network has been studied in \cite{samplecomplexity}, where difficulties in the form of an explosion of required sample size were shown, rather than algorithmic intractability. Similar results concerning the sample complexity were established in the framework of statistical query algorithms in \cite{chen2022hardness}. 

A key challenge is to increase the interpretability of deep learning algorithms, i.e., enabling the user to comprehend their decision-making, which is typically hindered due to the black-box behavior of deep learning \cite{ras2023explainable, Olah22MechInt, Kaestner23ExplAI}. Understanding the inner workings of deep learning also points out a way to establish trustworthy methods. Another approach relies on verifying the accuracy and correctness of deep learning methods without explicitly tracing internal computations \cite{Biondi20Certification, Zhang20Certification, Mirman21Certification, Katz17Certification}. However, the findings in \cite{problem8, Bastounis23Classification} indicate that certifying the accuracy and robustness of deep learning in the computability framework is challenging if at all possible, which poses challenges for future applications. A potential direction to cope with this issue was considered in \cite{analog}, where certain inverse problems that are not computable on digital computers were shown to be computable in a model of analog computation enabling implicit correctness guarantees in theory.

\subsection{Outline}
In Section \ref{sec:Def}, we introduce the applied formalisms, including computability theory and neural networks as the main workhorse of deep learning. We present our main results concerning Type 1 and Type 2 failure of computability in Section \ref{sec:failures}. We conclude in Section \ref{sec:Strategy} by studying strategies to cope with computability failures, whereby quantization is a main theme. The proofs of the theorems are provided in Appendix \ref{secA1}.

\section{Notation and Definitions}\label{sec:Def}
We first introduce some basic concepts and notation used in the following.

\subsection{Computability of Real Functions}
We begin by reviewing definitions from real-valued computability theory necessary for our analysis. For a more comprehensive overview, see, for instance, \cite{computableanalysisbook, Pour-El17Computability, avigadbrattka}. We also omit elementary topics of computability theory such as recursive functions and Turing machines. Here we refer the reader to \cite{computabilitybook}.

Previous results in applied computability on the real domain introduce many different, although partly equivalent, versions of computation and computability. We follow standard definitions introduced by Turing \cite{turing}.
\begin{Definition}
    A sequence of rational numbers $(q_k)_{k=1}^\infty\subset \Q$ is \emph{computable} if there exist recursive functions $a,b,s: \N \goto \N$ such that
    \begin{equation*}
        q_k = (-1)^{s(k)}\frac{a(k)}{b(k)}.
    \end{equation*}
    A rational sequence $(q_k)_{k=1}^\infty$ \emph{converges effectively} to $x \in \R$, if there exists a recursive function $e: \N \goto \N$ such that for all $k_0 \in \N$ and all $k \ge e(k_0)$
    \begin{equation*}
        \abs{x - q_k} \le \frac{1}{2^{k_0}}.
    \end{equation*}
    A real number $x \in \R$ is \emph{computable} if there exists a computable rational sequence $(q_k)_{k=1}^\infty$ converging effectively to $x$. Such a sequence is called a \emph{representation} (or a \emph{rapidly converging Cauchy name}) of $x$.
    We denote the set of all computable reals by $\R_c$.
\end{Definition}
Before defining computable real functions, we need to specify what it means for a real sequence to be computable.

\begin{Definition}
    A real sequence $(x_k)_{k=1}^\infty\subset \R$ is \emph{computable} if there exists a computable double-indexed rational sequence $(q_{k,\ell})_{k,\ell=1}^\infty$ such that, for some recursive function $e: \N \times \N \goto \N$ and all $k, \ell_0 \in \N$ and $\ell \ge e(k, \ell_0)$, we have
    \begin{equation*}
        \abs{x_k - q_{k,\ell}} \le \frac{1}{2^{\ell_0}}.
    \end{equation*}
\end{Definition}
\begin{Remark}
    All previous definitions can be extended to $\R^d$ and $\R_c^d$ with $d>1$ by requiring that each (one-dimensional) component or component-wise sequence is computable, respectively.
\end{Remark}

Out of the various definitions of a computable real function (see \cite[Appendix 2.9]{avigadbrattka} for an overview) we introduce two. Borel-Turing computability can be seen as the standard intuitive notion of computation, i.e., an algorithm approximating a given function to any desired accuracy exists. On the other hand, Banach-Mazur computability is the weakest common definition of computability meaning that a function is not computable in any usual sense if it is not Banach-Mazur computable.
\begin{Definition} \label{def:compfunction}
    Let $D \sube \R^d$. A function $f: D \goto \R_c^m$ is 
    \begin{enumerate}[label={(\arabic*)}]
        \item \emph{Borel-Turing computable} if there exists a Turing machine $M$ such that, for all $\x \in D\cap \R_c^d$ and all representations $(\vect{q}_k)_{k=1}^\infty$ of $\x$, the sequence $(M(\vect{q}_k))_{k=1}^\infty$ is a representation of $f(\x)$;
        \item \emph{Banach-Mazur computable} if for all computable sequences $(\x_k)_{n=1}^\infty \sube D\cap \R_c^d$ the sequence $(f(\x_k))_{k=1}^\infty$ is also computable.
    \end{enumerate}
\end{Definition}
\begin{Remark}\label{rm:CompFunc}
    For a Borel-Turing computable function, there exists a Turing machine taking a sequence of increasingly precise approximations of the input and producing increasingly accurate approximations of the output, whereas, a Banach-Mazur computable function only guarantees that it preserves the computability of real sequences. However, it is well known that all Borel-Turing computable functions are also Banach-Mazur computable, and computable functions in either sense are continuous -- that is, continuous on $\R_c$ with the inherited topology \cite{avigadbrattka}. For simplicity, we often refer to ``computable'' functions rather than ``Borel-Turing computable'', as this is our framework's standard version of computability. We explicitly specify whenever we apply the notion of Banach-Mazur computability.
\end{Remark}
The main theme of this paper revolves around the failure of algorithmic computations studied from the perspective of computability. In particular, we ask in what circumstances failures arise, and under what additional conditions they potentially can be avoided. Thereby, we associate in the real domain the more intuitive term ``algorithm'' with ``Borel-Turing computable function'' and we distinguish two cases of algorithmic failure:
\begin{itemize}
    \item We say that a problem suffers from \emph{Type 1 failure of computability} if it has no computable solver, that is, for any algorithm there exists an instance of the problem to which the algorithm provides an incorrect solution. 
    \item A problem is subject to \emph{Type 2 failure of computability} if a solver cannot be algorithmically found based on data, that is, for any learning algorithm $\Gamma$ there exists a problem instance $s$ such that for any dataset $\mathcal{X}$ the output $\Gamma(\mathcal{X})$ of the algorithm is not a correct solver of $s$.
\end{itemize}
Note that Type 1 as a special case of Type 2 failure is more fundamental since a (computable) solution cannot be learned from data if it does not exist. However, we show in Subsection \ref{sec:learn} that instances of Type 2 free of Type 1 failure exist in the context of training neural networks. Here, the problem instance is an unknown function and the goal of the learning algorithm is to find a neural network that represents the sought function based on samples.


\subsection{Neural Networks}
\label{sec:nndef}

In this paper, we restrict our attention to feedforward neural networks. For additional background on (deep) neural networks theory  -- usually referred to as deep learning -- we point to \cite{deeplearningbook}. We characterize neural networks by their structure, coined architecture, and the associated sequence of their parameters, i.e., their weight matrices and bias vectors.

\begin{Definition}
	Let $L \in \N$. An \emph{architecture} of depth $L$ is a vector $S := (N_0, N_1, \dots, N_{L-1}, N_L) \in \N^{L+1}$. A \emph{neural network} with architecture $S$ is a sequence of pairs of \emph{weight matrices} and \emph{bias vectors} $((A_\ell, \vect{b}_\ell))_{\ell=1}^L$ such that $A_\ell \in \R^{N_\ell \times N_{\ell-1}}$ and $\vect{b}_\ell \in \R^{N_\ell}$ for all $\ell = 1, \dots, L$. We denote the set of neural networks with architecture $S$ by $\mathcal{NN}(S)$ and the total number of parameters in the architecture by $N(S) := \sum_{\ell=1}^{L} (N_\ell N_{\ell-1} + N_\ell)$. 
\end{Definition}
\begin{Remark}
    Typically, we consider $b_L := \vect{0}$ and denote the input dimension $N_0 := d$. Also, throughout this paper, we focus on the case $N_L = 1$ for simplicity of presentation, even though the results can be reformulated for the general case.
\end{Remark}
The architecture and the parameter then induce the network's input-output function, the so-called realization.
\begin{Definition}
    For $\Phi \in \mathcal{NN}(S)$, $D \sube \R^{N_0}$, and $\sigma: \R \goto \R$ denote by $R_\sigma^D(\Phi): D \goto \R^{N_L}$ the \emph{realization} of the neural network $\Phi$ with \emph{activation} $\sigma$ and domain $D$, that is,
 \begin{equation*}
     R_\sigma^D(\Phi) := T_L \circ \sigma \circ \dots \circ \sigma \circ T_1|_{D},
 \end{equation*}
 where $\Phi = ((A_\ell, \vect{b}_\ell))_{\ell=1}^L$ and $T_\ell(\vect{x}) := A_\ell \vect{x} + b_\ell$, $\ell = 1, \dots, L$.
\end{Definition}

Parameters of neural networks are almost always the result of a learning algorithm. Let us briefly recapitulate the learning process since it pertains to our discussion on computability. A learning algorithm, typically some version of stochastic gradient descent, receives as input a dataset of sample pairs $(\vect{x}_i, \vect{y}_i)_{i=1}^n$, which are usually sampled from some underlying goal function $f$, that is, $f(\x_i) = \vect{y}_i$. 

Machine learning algorithms aim to find a function $\hat{f}$ approximating $f$. In deep learning, we typically take $\hat{f}:= R_\sigma^D(\Phi)$ for some neural network $\Phi$ with a fixed architecture $S$. The algorithm initializes the network with typically random parameters, followed by an iterative optimization of a loss function $\mathcal{L}: \mathcal{NN}(S) \times \left(\R^{N_0} \times \R^{N_L}\right)^n \goto \R$. A popular choice is the mean square error
\begin{equation*}
    \mathcal{L}(\Phi, \x_1, \vect{y}_1, \dots, \x_n,\vect{y}_n) := \frac{1}{n} \sum_{i=1}^n \norm[]{R_\sigma^D(\Phi)(\x_i) - \vect{y}_i}^2,
\end{equation*}
where $\norma{\cdot}$ indicates the Euclidean norm throughout the paper. Note that there exists a homeomorphism between neural networks with architecture $S$ and their parameter space $\R^{N(S)}$, that is, $\mathcal{NN}(S) \approx \R^{N(S)}$. Thus we can view the learning algorithm as a Borel-Turing computable function $\Gamma: \R_c^{n(N_0+N_L)} \goto \R_c^{N(S)}$, which receives a set of samples $(\vect{x}_i, \vect{y}_i)_{i=1}^n$ and returns weights and biases of the optimized network. More precisely, given rational sequences representing the real data the algorithm produces a sequence representing weights and biases. Therefore, the parameters of the resulting neural network will be computable numbers. To distinguish between networks with real and computable parameters, we introduce the notation $\mathcal{NN}_c(S)$ for the set of neural networks with architecture $S$ and parameters in $\R_c$. Again, we can establish a homeomorphism between neural networks $\mathcal{NN}_c(S)$ and their parameters $\R_c^{N(S)}$ with the inherited topology. An important property of any network in $\mathcal{NN}_c(S)$ is that its realization is a computable function provided that the applied activation function is computable.

Finally, for ease of presentation of our analysis concerning Type 2 failure of neural networks, we apply the following concise form to describe sampled data sets. 
\begin{Definition}
    Given $n \in \N$, $f: \R^d \to \R^m$ and $D \sube \R^d$, we denote by $\mathcal{D}_{f,D}^n$ the set of all \emph{datasets} of size $n$ generated from $f$ on the input domain $D$, that is,
    \begin{equation*}
        \mathcal{D}_{f,D}^n := \left\{ (\x_i, f(\x_i))_{i=1}^n \in (\R^d\times\R^m)^n \separ \x_i \in D, i = 1,\dots,n \right\}.
    \end{equation*}
    For a neural network $\Phi \in NN(S)$ with activation $\sigma$ we denote for short $\mathcal{D}_{\Phi,D}^n := \mathcal{D}_{R_\sigma^D(\Phi),D}^n$.
\end{Definition}

\section{Computability Limitations}
\label{sec:failures}

Our first goal is to study and actually establish Type 1 and Type 2 failure for general problem descriptions, namely in classification and learning. Subsequently, we will analyze approaches to cope and ideally lessen the derived failures without compromising the generality of the considered problems. The two settings -- classification and learning -- are chosen because of their importance in deep learning. Classification is one of the main objectives tackled by deep learning, whereas learning is one key component of the method. Having reliable, flexible, and universal learning algorithms hugely benefits the applicability of deep learning in various fields. Thus, classification and learning are suitable choices to highlight the consequences of computability failures.

\subsection{Type 1 Failure in Classification}\label{sec:decide}

We study classification problems from the viewpoint of computability theory and show that we can find Type 1 failure of computability. A classification problem is modeled by a function 
\begin{equation*}
    f: D \to \{1,\dots,C\}, \qquad D \subset \R^d,\quad  C \in \N,
\end{equation*} 
that assigns each input $\x \in D$ a corresponding class $c \in \{1,\dots,C\}$. A typical example is image classification where the input domain $D$ is for instance given by $D=[0,256]^{h\times w}$ with $[0,256]$ and $h,w \in \N$ encoding color and size (height and width) of an image, respectively. The range $[0,256]$ may also be quantized so that a discrete set such as $\{0,\dots,256\}$ represents the input domain. First, we explore the continuous setting before turning to quantized problems in Subsection \ref{sec:quantized}. The properties of the input domain play a crucial role in establishing Type 1 failure. It turns out that quantization alleviates Type 1 limitations arising in the continuous case. 

As described in Subsection \ref{sec:nndef}, the goal of deep learning is to learn a function $\hat{f}: D \cap \R_c^d \to \{1,\dots,C\}$ based on samples $(\x_i, f(\x_i))_{i=1}^n \sube D \times \{1,\dots, C\}$ such that $\hat{f}$ is close to $f$ with respect to a suitable measure. Hence, a crucial question is whether $\hat{f}$ can be obtained from an algorithmic computation given a specific closeness condition. Equivalently, this question can be expressed in terms of (semi-)decidability on the input domain of $f$ in $\R^d$ if $\hat{f}$ is expected to exactly emulate $f$, i.e., $\hat{f} = f|_{\R_c^d}$. 
\begin{Definition}
   A set $A \sube D \cap \R_c^d$ is
    \begin{itemize}
        \item \emph{decidable} in $D$, if its indicator function $1_A: D \cap \R_c^d \goto \R_c$ is computable;
        \item \emph{semi-decidable} in $D$, if there exists a computable function $f: D' \goto \R_c$, $D' \sube D \cap \R_c^d$, such that $A \sube D'$ and $f = 1_A|_{D'}$.
    \end{itemize}
\end{Definition}
\begin{Remark}    
    The notion of (semi-)decidability can be explicitly expressed via algorithms in the following way.
    The set $A$ is Borel-Turing decidable in $D$ if there exists a Turing machine $M$ taking as inputs representations of $\x \in D \cap \R_c^d$ which correctly identifies after finitely many iterations whether $\x \in A$ or $\x \in D\setminus A$. Similarly, $A$ is semi-decidable in $D$ if there exists a Turing machine $M$ taking as inputs representations of $\x \in D$ which correctly identifies (after finitely many iterations) every input $\x \in A$, but which may run forever for $\x \in D\setminus A$.
\end{Remark}
Recall that computable functions are necessarily continuous on $\R_c^d$. However, indicator functions are discontinuous on $\R_c^d$ (excluding the trivial cases $\R_c^d$ and $\emptyset$). Thus, only sets of the type $\R_c^d\cup B$, $B \subset \R^d\setminus \R^d_c$, are decidable in $\R^d$. Therefore, decidability in $\R^d$ is a very restrictive notion that typically will not be satisfied by a classifier $f$. Regarding semi-decidability, the following equivalence is immediate due to the discrete image of $f$. In particular, the proposition provides a necessary condition for learning a perfect emulator $\hat{f} = f|_{\R_c^d}$ since computability is a prerequisite for learnability. 
\begin{Proposition} \label{prop:decide}
    Let $D \sube \R^d$ and consider $f:D \to \{1,\dots,C\}$. Then, $f|_{\R_c^d}$ is computable if and only if each set $f^{-1}(i)$, $i=1,\dots,C$, is semi-decidable in $D$.
\end{Proposition}
\begin{Remark}
        Note that if each set $f^{-1}(i)$, $i=1,\dots,C$, is semi-decidable in $D$, then also each set $f^{-1}(i)$ is decidable in $D$. Therefore, for $f$ to be computable, $D$ has to have a specific structure. In particular, if $D \cap \R_c^d$ is a connected set homeomorphic to $\R^d_c$, e.g., $D=(0,1)^d$, then $f$ is not computable unless it is constant on $D \cap \R_c^d$. However, in this case, the classification problem is itself trivial. Thus, a necessary condition for computability is that the sets $f^{-1}(i)$ are separated to a certain degree. As a simple example for this type of problems consider $D=f^{-1}(1) \cup f^{-1}(2)$ with  $C=2$, $f^{-1}(1)=(0,1)$ and $f^{-1}(2)=(2,3)$. Here, a simple check of whether a given input is smaller or larger than 1.5 is sufficient to determine the associated class of the input.  
\end{Remark}
These observations can be summarized in the following conclusion, where $\text{dist}(A,B)$ denotes the set distance $\text{dist}(A,B) = \inf_{a \in A, b \in B} \norm[]{a-b}$.
\begin{Proposition}\label{prop:Connected}
    Let $f: D \goto \{1, \dots, C\}$ be a function such that there exists $i \neq j \in \{1, \dots, C\}$ with $\text{dist}(f^{-1}(i), f^{-1}(j)) = 0$. 
    Then, $f|_{\R_c^d}$ is not computable.
\end{Proposition}
\begin{Remark}
    A desired property in certain applications is identifying inputs not belonging to the known classes. For instance, unknown classes or erroneous inputs, which were not determined beforehand or cannot be unequivocally assigned to the known classes, respectively, may be part of the feasible input set. 
    Formally, we can express this setting by a classifier 
    \begin{equation*}
        \hat{f^\prime}: D^\prime \to \{1,\dots,C+1\}, \quad D \subset D^\prime \subset \R^d,
    \end{equation*}
    where $D^\prime$ is connected, such that $\hat{f^\prime}(\x) = f(\x)$ for $\x \in D$ and $\hat{f^\prime}(\x) = C+1$ otherwise. However, we immediately observe that $f^\prime$ is not computable and the desired property cannot be achieved.
\end{Remark}
The results in Proposition \ref{prop:decide} and \ref{prop:Connected} imply that Type 1 computability failure is unavoidable in sufficiently general classification problems. Our next step is to study whether algorithmic solvability can be expected in the absence of Type 1 failure.

\subsection{Type 2 Failure in Deep Learning} \label{sec:learn}

We now focus on Type 2 failure of computability, i.e., situations where a computable approximator may exist but cannot be algorithmically found based on data. We explore this phenomenon in the context of deep learning, within a general framework by studying the learnability of functions that can be represented by a neural network from data, independently of the concrete application. This includes any instance of deep learning where Type 1 computability failure does not arise, going beyond the previous context of classification.

The following theorem states that for any learning algorithm, there exist functions representable by computable neural networks (i.e., not suffering from Type 1 failure) that the algorithm cannot learn from data. This implies that there is no universal algorithm for training neural networks (based on data), even when a correct solving network exists, as deep learning suffers from Type 2 failure of computability. More precisely, for any learning algorithm $\Gamma$, there exists a computable neural network $\Phi$ such that given any training data set generated from $\Phi$, the algorithm $\Gamma$ cannot reconstruct any neural network with the same realization as $\Phi$.

\begin{Theorem}
	\label{learning}
	Let $\sigma: \R \goto \R$ be a Lipschitz continuous, but not affine linear activation function, such that $\sigma|_{\R_c}$ is Banach-Mazur computable. Let $S = (d, N_1, \dots, N_{L-1}, 1)$ be an architecture of depth $L \ge 2$ with $N_1 \ge 3$. Let $D \sube \R_c^d$ be bounded with a nonempty interior.
	
	Then, for all $\varepsilon > 0$, $n \in \N$, and all Banach-Mazur computable functions $\Gamma: (\R_c^d\times \R_c)^n \goto \R_c^{N(S)}$ there exists $\Phi \in \mathcal{NN}_c(S)$ such that for all $\mathcal{X} \in \mathcal{D}_{\Phi,D}^n$ and all $\Phi^\prime \in \mathcal{NN}_c(S)$ with $R_\sigma^D(\Phi^\prime) = R_\sigma^D(\Phi)$ we have
	\begin{equation}\label{eq:Weigtdist}
		\norm[2]{\Gamma(\mathcal{X}) - \Phi^\prime} > \varepsilon.
	\end{equation}
\end{Theorem}
\begin{Remark}
    The assumption that $\sigma$ is not affine linear excludes none of the commonly used activations such as ReLU, tanh, or sigmoid. Only functions like $\sigma(t) = at+b$ are not permitted. Moreover, the computability assumption concerning $\sigma$ is not restrictive since it guarantees a computable realization, a prerequisite for subsequent algorithmic evaluation in usage. 
\end{Remark}
As a consequence, we cannot reconstruct the original input-output function.
\begin{Corollary}\label{cor:nonLearning}
	Under the assumptions of Theorem \ref{learning}, there exists no Banach-Mazur computable function $\Gamma: (\R_c^d\times \R_c)^n \goto \R_c^{N(S)}$ for $n \in \N$ such that for all $\Phi \in \mathcal{NN}_c(S)$ there exists a dataset $\mathcal{X} \in \mathcal{D}_{\Phi,D}^n$ satisfying
	\begin{equation*}
		R_\sigma^D(\Gamma(\mathcal{X})) = R_\sigma^D(\Phi).
	\end{equation*}
\end{Corollary}
Note that the distance in \eqref{eq:Weigtdist} in Theorem \ref{learning} is shown in the weight space of neural networks and it is known that networks with weights far apart can still represent functions close together \cite{topology}. Despite the failure of exact reconstruction obtained in Corollary \ref{cor:nonLearning}, a neural network with the same architecture, which approximates the desired realization to an arbitrary degree, might still exist. We will return to this observation in Subsection \ref{sec:approximatelearning}, where we consider strategies for coping with non-computability by relaxing the exactness condition.

\section{Strategies for Failure Circumvention}\label{sec:Strategy}
Can we overcome Type 1 and Type 2 limitations described in the previous section? We analyze different approaches to either reformulate or relax the tackled problems thus making them less amenable to computability failures. In particular, we explore two strategies. First, we study the effect of incorporating a reasonable error mode (depending on a given task) in the computation. Subsequently, we investigate the impact of moving the problem from the real to a discrete space via quantization.

\subsection{Error Control Strategies}
The requirement of exact emulation $\hat{f} = f|_{\R_c^d}$ of a classification function $f$ or exact reconstruction of neural networks as analyzed in Section \ref{sec:failures} may be too strict. In a practical setting, errors may be unavoidable or even acceptable to a certain degree. In particular, approximation of $f|_{\R_c^d}$ via $\hat{f}$ or reconstructing an approximate network based on an appropriate metric is a simpler task than exact emulation or reconstruction, respectively. However, in both cases, we certainly would like to have guarantees either in the form of a description of inputs that lead to deviations from the ground truth or via worst/average case error bounds. Whether and to what degree such guarantees are achievable is the subject of the following analysis.

\subsubsection{Computable Unpredictability of Correctness in Type 1 Failure}\label{sec:flag}

A key observation in classification was that Type 1 failure, i.e., the non-semi-decidability of the classes, is closely associated with the decision boundaries of the classes. Informally speaking, the semi-decidability of classes hinges on the ability to algorithmically describe the decision boundary so that inputs on the decision boundary can be properly classified. Therefore, identifying these critical inputs or indicating that the computation for a given input may be inaccurate would certainly be beneficial. Is it possible to implement this identification -- a so-called exit flag -- algorithmically? In a bigger picture, related questions were already raised in different contexts such as general artificial intelligence by Daniel Kahneman \cite{LexFridman} or robotics by Pieter Abbeel \cite{Abbeel}: Can we expect algorithms (powering autonomous agents) to recognize whenever they cannot correctly solve a given task or instance enabling them to ask (a human) for help instead of executing an erroneous response? In other words: `Do they know when they don't know?' \cite{cheng2024aiassistantsknowdont, ren2023robots} We immediately observe that this problem depends on an exit flag computation in the computability framework. Hence, by studying exit flag computations, we can also theoretically assess the feasibility of automated help-seeking behavior of autonomous agents in certain scenarios.

To study the posed question we do not restrict ourselves to classification functions but consider a slightly more general framework. We formalize the problem for general real-valued functions $f:D \to \R$, $D \sube \R^d$. Assume we are given a computable function $\hat{f}: D_c \goto \R_c$, $D_c = D \cap \R_c^d$, typically constructed by an algorithmic method to approximate $f$. Our aim is to algorithmically identify inputs $\x\in D_c$ such that $\hat{f}$ satisfies
\begin{equation*}
    \|f(\x) - \hat{f}(\x)\| < \varepsilon \qquad \text{ for given } \varepsilon>0.
\end{equation*}
In other words, we ask if there exists an algorithm (a Borel-Turing computable function) $\Gamma_\varepsilon: D_c \to \R_c$ such that 
\begin{equation}\label{eq:exit-flag}
    \Gamma_\varepsilon(\x) = 
    \begin{cases} 
        1, & \text{ if } \|f(\x) - \hat{f}(\x)\|, \\                              
        0, &  \text{ otherwise. }
    \end{cases}   
\end{equation}
One can also further relax the complexity of the task by demanding that an algorithm $\Gamma^+_\varepsilon$ only identifies inputs $\x$ for which $\hat{f}(\x)$ satisfies the $\varepsilon$-closeness condition in \eqref{eq:exit-flag}, but does not necessarily indicate when it does not hold. For instance, $\Gamma^+_\varepsilon(\x)$ may either output zero or not stop the computation in finite time on the given input $\x$ if $\|f(\x)- \hat{f}(\x)\| \geq \varepsilon$. If $f$ is a computable function, we can construct $\Gamma_\varepsilon$ for any $\varepsilon>0$. Hence, more interesting problems arise when $f$ is not computable. However, choosing $\varepsilon$ large enough, certainly still entails the existence of $\Gamma_\varepsilon$ if $f$ is, for instance, a bounded function. Therefore, the relevant cases are associated with non-computable $f$ and appropriately small $\varepsilon$.
By associating $\Gamma_\varepsilon$ and $\Gamma^+_\varepsilon$ with classification functions, we can apply Proposition \ref{prop:decide} to derive the following result.
\begin{Proposition} \label{prop:flag}
    Let $f:D \to \R$, $D \sube \R^d$ and assume that $\hat{f}: D_c \to \R_c$, $D_c = D\cap \R_c^d$, is a computable function. Define for $\varepsilon >0$ the set
    \begin{equation*}
        D^<_\varepsilon := \left\{\x \in D_c ~\big|~ \|f(\x)- \hat{f}(\x)\|< \varepsilon\right\}    
    \end{equation*}
    Then the following holds:
    \begin{enumerate}
        \item The function $\Gamma_\varepsilon : D_c \to \R_c$ given by 
        \begin{equation*}
            \Gamma_\varepsilon(\x) = 
            \begin{cases} 
                1, & \text{ if } \|f(\x)- \hat{f}(\x)\|< \varepsilon, \\ 
                0, &  \text{ otherwise, }
            \end{cases}  
        \end{equation*} 
        is computable if and only if $D^<_\varepsilon$ is decidable in $D$.
        \item A computable function $\Gamma^+_\varepsilon : D_c' \to \R$, $D_c' \sube D_c$, such that $D^<_\varepsilon \sube D_c'$ and $\Gamma^+_\varepsilon = \Gamma_\varepsilon|_{D_c'}$ exists if and only if $D^<_\varepsilon$ is semi-decidable in $D$.
    \end{enumerate}
\end{Proposition}
\begin{Remark}
        Depending on the context, we might be more interested in finding the set of inputs where $\hat{f}$ fails rather than succeeds. This would lead us to the analogous observation that the semi-decidability of the set 
        \begin{equation*}
            D^{\ge}_\varepsilon = \left\{\x \in D_c ~\big|~ \|f(\x)- \hat{f}(\x)\|\ge \varepsilon\right\} 
        \end{equation*}
        determines computability of $\Gamma^-_\varepsilon: D_c' \to \R$, $D_c' \sube D_c$, given by  $\Gamma^-_\varepsilon = \Gamma_\varepsilon|_{D_c'}$ with $D^{\ge}_\varepsilon \sube D_c'$.
\end{Remark}
In the case of a connected input domain, the application of Proposition \ref{prop:Connected} yields the following result. Informally, it is a direct consequence of the already mentioned fact that only trivial subsets of $\R_c^d$ are decidable.
\begin{Corollary} \label{cor:flag}
    Under the conditions of Proposition \ref{prop:flag}, additionally assume that $D$ is connected. Then an approximator $\hat{f}$ such that $D^<_\varepsilon$ is decidable exists if and only if $f$ can be computably approximated with precision $\varepsilon$, that is, if there exists a computable function $\Tilde{f}$ such that
    \begin{equation*}
        \|f - \Tilde{f}\|_\infty < \varepsilon.
    \end{equation*}
\end{Corollary}
\begin{Remark}
    As the following example shows, a similar statement does not hold if $D^<_\varepsilon$ is only assumed to be semi-decidable. Consider the sign function $\text{sgn}: \R \goto \R$, which is non-continuous on $\R_c$ and therefore non-computable, and take $\widehat{\text{sgn}}(x) = \frac{2}{\pi} \text{arctan}(x)$. For a given $\varepsilon > 0$ we can computably construct intervals $(-\infty,-x_0)$ and $(x_0, \infty)$ where $\abs{\text{sgn}(x) - \widehat{\text{sgn}}(x)} < \varepsilon$, i.e., $D^<_\varepsilon = (-\infty,-x_0) \cup (x_0, \infty)$ is semi-decidable. In fact, we can adjust the approximator to achieve the desired precision on a given interval $(x_0, \infty), x_0>0$. However, due to the discontinuity at $0$, no computable function approximating $\text{sgn}$ on the entire real line with precision $\varepsilon < 1$ exists.    
\end{Remark}
These results also directly apply to the classification setting as a special case of the considered framework. By design, the classification setting even allows for stronger statements regarding the magnitude of the error $\varepsilon$. In particular, requiring precision of $\varepsilon < \frac{1}{2}$ for the approximator $\hat{f}$ of a classifier $f$ mapping to $\{1, \dots, C\}$ is equivalent to requiring exact emulation $\hat{f} = f|_{\R_c^d}$. However, this scenario was already covered in Subsection \ref{sec:decide}, where Type 1 failure was established. Hence, we can conclude that exit flag computations may be beneficial in certain situations but it is not appropriate to tackle Type 1 failure in classification.

Instead of analyzing individual inputs, one could also derive global guarantees for the approximator $\hat{f}$. For instance, one could determine the magnitude of the failure set for some appropriate measure, i.e., assess the likelihood of errors for some given input domain. Although this approach is interesting it has two main limitations for our intended goals. First, there does not exist an acknowledged algorithmic notion covering this framework that allows for theoretical studies - we refer to Appendix \ref{app:semi} for more details and possible concepts to tackle this problem. Second and more importantly, this strategy typically leads to a global quantitative measure of failure whereas we are interested in local guarantees for a given input. In essence, this framework does not tackle the limitations of individual inputs but provides further information on the entire input domain of the general problem.

Finally, we want to highlight that Proposition \ref{prop:flag} and Corollary \ref{cor:flag} imply that the algorithms envisioned by Kahneman and Abbeel can not be realized on digital hardware. For any algorithm $\Gamma$ tackling a problem described by a non-computable function, there exist instances that $\Gamma$ answers incorrectly and there is no algorithm $\Gamma_{\text{exit}}$ that recognizes the failure for all erroneous instances. Hence, the answer to `Do they know when they don't know?' from the digital computing perspective is negative in certain scenarios.

\subsubsection{Problem Relaxation for Type 2 Failure}
\label{sec:approximatelearning}
In contrast to classification, relaxing the exact reconstruction requirement yields learning benefits. The main advantage is that the solution set of networks connected to approximate reconstruction is noticeably larger enabling algorithmic approaches to perform the previously unattainable reconstruction task. Type 2 failure does not arise in our learning setting on the training data if a certain approximation error is permitted. 
\begin{Theorem}\label{thm:posLearning}
    Let $\sigma: \R \goto \R$ be such that $\sigma|_{\R_c}$ is computable and let $S = (d, N_1, \dots, N_{L-1}, 1)$ be an architecture.
	
	Then, for all $\varepsilon > 0$ and $n \in \N$, there exists a computable function $\Gamma: (\R_c^d\times \R_c)^n \goto \R_c^{N(S)}$ such that for all $\Phi \in \mathcal{NN}_{c}(S)$ and  $\mathcal{X} \in \mathcal{D}_{\Phi,\R_c^d}^n$ we have 
    \begin{equation}\label{eq:condRec}
        \big| R_\sigma^{\R_c^d}\left(\Gamma(\mathcal{X})\right)(\x) - y \big| < \varepsilon \quad \text{ for } (\x,y) \in \mathcal{X}.	
    \end{equation}
\end{Theorem}
Although Theorem \ref{thm:posLearning} provides a positive computability result concerning learning it still has certain limitations: 
\begin{itemize}
    \item The algorithm constructed to prove Theorem \ref{thm:posLearning} serves only for theoretical analysis. It is typically not efficiently translatable into a practically usable one in a generic problem setting. Thus, a relevant and open question is whether more applicable learning algorithms can be constructed with similar guarantees.
    \item The learning algorithm $\Gamma$ derived from Theorem \ref{thm:posLearning} presupposes a fixed accuracy parameter $\varepsilon$ and dataset size $n$, i.e., for different choices of these parameters a separate learning algorithm needs to be constructed. 
    \item The reconstruction guarantee only holds on the training data. However, given access to the training data $\mathcal{X}$, the posed task could be solved without constructing a neural network since one could explicitly implement an algorithm that on the input of $(\x,y) \in \mathcal{X}$ returns $y$. Constructing a neural network with some prescribed realization on the training data is expected to yield a network that performs 'reasonably well' on unseen data. The learning and evaluation, as given in \eqref{eq:condRec}, should ideally be performed on different data to ensure this hypothesis.
\end{itemize}
We cannot alleviate the first but resolve the second issue. Indeed, the proof of Theorem \ref{thm:posLearning} implies that one could generalize the algorithm $\Gamma$ by requiring it to take (computable) $\varepsilon$ and $n$ as additional inputs. In particular, the key step of the proof relies on an enumeration argument that can be extended to incorporate $\varepsilon$ and $n$ as well. In a slightly different setting, one can connect the desired accuracy with the input dimension and sample complexity \cite{samplecomplexity}. Hence, the final point is left to consider. By imposing further conditions on the admissible networks, we can ensure certain generalization capabilities of the networks.
\begin{Theorem}\label{thm:posLearning2}
    Let $\sigma: \R \goto \R$ be such that $\sigma|_{\R_c}$ is computable and Lipschitz continuous. Fix $A_{\text{max}} \in \N$ and let $S = (d, N_1, \dots, N_{L-1}, 1)$ be an architecture. 
	
	Then, for all $\varepsilon > 0$ and $n \in \N$, there exist computable functions $\Gamma: (\R_c^d\times \R_c)^n \goto \R_c^{N(S)}$ and $\Psi: \R_c^{N(S)} \to \R_c^+$ such that for all $\Phi = ((A_\ell, \vect{b}_\ell))_{\ell=1}^L \in \mathcal{NN}_{c}(S)$ with $\max_\ell \|A_\ell\|_{\max} \leq A_{\text{max}}$ and  $\mathcal{X} \in \mathcal{D}_{\Phi,\R_c^d}^n$ we have 
    \begin{equation}\label{eq:GenGuarantee}
        \left| R_\sigma^{\R_c^d}\left(\Gamma(\mathcal{X})\right)(\x) - y \right| < \varepsilon \quad \text{ for } (\x,y) \in \mathcal{X}^{\Phi}_{\Psi(\Gamma(\mathcal{X}))}, 
    \end{equation}
    where $\mathcal{X}^{\Phi}_r := \{(\x, \Phi(\x)) \separ \x \in \bigcup_{(\x_i,y_i) \in \mathcal{X}} B_{r}(\x_i)\}$ 
\end{Theorem}
\begin{Remark}
    For specific classes of neural networks, a uniform lower bound $\delta$ on $\Psi(\Phi)$ can be established (over the set of networks $\Phi$). Therefore, it is possible to provide approximate reconstruction guarantees for the entire input domain $D$ if $\mathcal{X}^\Phi_{\Psi(\Gamma(\mathcal{X}))}$ covers $D$. For instance, given an equidistant data grid on $D$ with width $\delta$ an approximate reconstruction guarantee holds for the whole domain.
\end{Remark}
It is straightforward to convince oneself that the imposed conditions on data and networks are necessary. Guarantees of the form \eqref{eq:GenGuarantee} cannot be provided without the assumptions. One feasible step beyond the considered setting is to consider arbitrary training data which is not necessarily sampled from a neural network but some underlying ground truth function $g:\R^d \to \R$, i.e., the ground truth is not necessarily realizable by a neural network (with given architecture). In its current form, Theorem \ref{thm:posLearning} and \ref{thm:posLearning2}, entail the existence of networks with approximate realization (to the ground truth network), i.e., the remaining task is algorithmically finding it. For arbitrary training samples of the form $(\x_i,g(\x_i))_{i=1}^n \subset (\R^d\times \R)^n$, one would need to first establish the existence of a neural network approximating the training data, i.e., the ground truth $g$, to a sufficient degree with the prescribed architecture. However, by taking into account the expressivity of the architecture, it is feasible to provide computability guarantees in this setting as well under sufficient conditions on the data and class of ground truth functions. Another strategy is to ask for an optimal network minimizing the error for some loss function on the training data. In \cite{yunseok}, it was shown that this scenario indeed suffers from Type 2 failure of computability, even for simple networks. A strategy to nevertheless avoid Type 2 failure could consist of relaxing the optimality requirement on the desired network but this problem warrants further investigations.   

Finally, we want to highlight the key differences between neural network training in our setting and general classification tasks. Why do the problems entail different degrees of computability failures? The crucial observation is that our considered model of neural networks does not cover the typical (more general) neural network model applied in classification. A classifier $\hat{f}$ as considered in Subsection \ref{sec:decide} is a discontinuous function, whereas a neural network, assuming continuous activation, possesses a continuous realization. In the context of neural network classification, a classifier $\hat{f} = \hat{f}_1 \circ R_\sigma^D(\Phi)$ is composed of a neural network $\Phi$ -- the so-called feature map -- and a (discontinuous) function $\hat{f}_1$ mapping from the features to the classes. The Type 1 failure of computability appears at $\hat{f}_1$, with no a priori restriction on computability of $\Phi$, i.e., Type 2 failure on this level is avoidable.

\subsection{Quantization Strategies} \label{sec:quantized}
Transformation of sequences approximating real numbers with arbitrary precision is the general paradigm describing digital computation on real numbers introduced by Turing himself \cite{turing}. It provides the tools to study the capabilities and limitations of perfect digital computing. Due to real-world constraints, a simplified and more applicable quantized model is typically employed to implement digital computations in practice. In quantization, real numbers are approximated by a discrete set of rationals. For instance, under fixed-point quantization, real numbers are replaced by rational numbers with a fixed number $k$ of decimal places in some base system $b$, i.e., algorithms strictly operate on the set $b^{-k}\Z$. Thus, assuming fixed-point quantization we can restrict the analysis without loss of generality to classification problems on $\Z^d$ as well as neural networks with integer parameters and data. The crucial difference between integer computability and the previously considered real-valued framework is the feasibility of exact computations so that approximative computations are not inherently necessary. The concept of exact algorithmic computations on integers is described by \textit{recursive functions} (which the previously considered framework of Borel-Turing computable functions extends to the real domain); we refer to \cite{computabilitybook} for more details on recursive functions and classical computability on discrete sets. 

Quantizing the parameter does not lead to a critical degradation of expressive power in neural networks. In particular, in the limit, the capabilities of quantized and real networks align \cite{bolcskei2019optimal, elbrachter2021deep, haase2023lower}. Nevertheless, we show next that in the context of the simplest quantization technique, namely fixed-point quantization, the computability limitations introduced in Section \ref{sec:failures} are alleviated to a certain degree. In particular, we establish that both Type 2 and Type 1 failures of computability are mainly resolved in this setting. Therefore, a crucial question is whether the quantization process for a given problem, if necessary, can be carried out algorithmically without computability failure.

\subsubsection{Computability of Quantized Deep Learning}
We show that under fixed-point quantization, the negation of Theorem \ref{learning} holds. That is, an algorithm exists that can re-learn the exact realization of neural networks of fixed architecture on the training data. To that end, we introduce the set of neural networks with integer parameters. 
\begin{Definition}
    Let $S$ be an architecture. Denote by $\mathcal{NN}_{\Z}(S) \sube \mathcal{NN}_c(S)$ the set of neural networks with architecture $S$ and parameters in $\Z$. 
\end{Definition}
Now, we can formulate the exact statement about re-learning neural networks.
\begin{Theorem}\label{thm:QuantLearning}
    \label{quantizedlearning}
    Let $S = (d, N_1, \dots, N_{L-1}, 1)$ be an architecture and let $\sigma: \R \goto \R$.

    Then, for all $n \in \N$ there exists a recursive function $\Gamma: (\Z^d\times \Z)^n \goto \Z^{N(S)}$ such that for all $\Phi \in \mathcal{NN}_{\Z}(S)$ there exists a dataset $\mathcal{X} \in \mathcal{D}_{\Phi, \Z^d}^n$ with
    \begin{equation*}
        R_\sigma^{\Z^d}\left(\Gamma(\mathcal{X})\right) = R_\sigma^{\Z^d}(\Phi).	
    \end{equation*}
\end{Theorem}
Despite the positive result, we can still raise two main limitations of Theorem \ref{thm:QuantLearning}. First, the theory-to-practice gap in learning algorithms remains an (open) issue as in the previous analysis. Second, Theorem \ref{thm:QuantLearning} only guarantees the existence of a dataset enabling reconstruction. Can we improve the statement to ensure reconstruction for any dataset satisfying some (weak) conditions? Before answering the question we want to point out a well-known fact: One cannot expect an exact reconstruction of a neural network's realization on the entire input domain based on an arbitrary but finite set of data samples in general. On the one hand, networks with different architecture or parameters may realize the same function, on the other hand, networks, whose outputs agree on some inputs, may wildly diverge in their realization \cite{topology}. 
\begin{Theorem}
    \label{quantizedlearningV2}
    Let $S = (d, N_1, \dots, N_{L-1}, 1)$ be an architecture and let $\sigma: \R \goto \R$ be an activation function such that $\sigma |_{\Z}$ is a recursive function.

    Then, for all $n \in \N$, there exists a recursive function $\Gamma: (\Z^d\times \Z)^n \goto \Z^{N(S)}$ such that for all $\Phi \in \mathcal{NN}_{\Z}(S)$ and $\mathcal{X} \in \mathcal{D}_{\Phi,\Z^d}^n$ we have 
    \begin{equation*}
        R_\sigma^{\Z^d}\left(\Gamma(\mathcal{X})\right)(\x) = R_\sigma^{\Z^d}(\Phi)(\x) \quad \text{ for all } \x \in \mathcal{X}.	
    \end{equation*}
\end{Theorem}
\begin{Remark}
    From a data-centric perspective, Theorems \ref{quantizedlearning} and \ref{quantizedlearningV2} describe edge cases, i.e., guarantees applicable to any test data in the former (at the cost of flexibility in the training data) and guarantees for any training data (at the cost of flexibility in the test data). Similar to Subsection \ref{sec:approximatelearning}, one can extend Theorem \ref{quantizedlearningV2} by imposing regularity conditions on the considered networks or relaxing the exactness condition to provide generalization bounds. Furthermore, if quantization yields a finite input domain one can trivially control the data set sizes to ensure exact generalization on the considered domain.        
\end{Remark}

\subsubsection{Computability of Quantized Classification}

Turning our attention to classification and Type 1 failure, we can distinguish between two cases. In many applications, classification is performed on a bounded domain $D$ such as in image classification described in Subsection \ref{sec:decide}. Hence, the quantized version of the input domain is finite so any integer-valued function on the quantized domain is computable - one can encode the input-output pairs directly in an algorithm. 
\begin{Proposition} \label{prop:quantizedclassification}
    Let $D \subset \Z^d$ and $f: D \to \{1,\dots, C\}$. If $D$ is bounded, then $f$ is recursive.
\end{Proposition}
In contrast, unbounded sets typically do not appear in practical quantized classification problems since they correspond to working on an infinite domain. However, in such a scenario we cannot provide formal guarantees on the computability of classifiers. Similar to the real case, the task reduces to classical (semi-)decidability of (infinite) sets of integers, which is not algorithmically solvable in general \cite{computabilitybook}. Hence, in both real and quantized classification Type 1 failure may arise due to non-(semi)-decidable sets in the respective frameworks. Nevertheless, the occurrence of Type 1 failure appears to diverge in the frameworks. Although it is intricate to derive a formal proof to back this statement, informally it is motivated by the observation that non-semi-decidability is a more severe drawback in the real domain. For instance, we have seen that non-trivial sets on $\R^d_c$ are not decidable whereas such a strong claim is not valid for the integer domain. 

\subsubsection{Non-Computability of the Quantization Function}
We have shown that quantization circumvents or at least mitigates Type 1 and Type 2 failure. Does it imply that in a real-world setting, where we typically compute with digital computers in the quantized model, Type 1 and Type 2 failures do not arise? Not necessarily, it depends on the ground truth problem. If the ground truth is itself quantized, then it typically can be directly translated into the quantized model of a digital computer and in principle algorithmically solved. In contrast, a ground truth problem on a continuous domain must first be converted into an appropriate quantized problem that approximates the original problem. Therefore, it would be desirable to provide computable guarantees that this approximation is close to the original. However, for a non-computable ground truth problem such an algorithmic verification contradicts its non-computability, as the ground truth problem could then be algorithmically computed/approximated using the verifier. Thus, quantization itself is a non-computable task, which is a direct consequence of Proposition \ref{prop:flag}.
\begin{Proposition} \label{prop:quantizationfunction}
    Let $f: \R \goto \R$ such that $f|_{\R_c}$ is not computable and define for all $x \in \R_c$
    \begin{equation*}
        \hat{f}(x) := f\left(\lceil x - \tfrac{1}{2} \rceil \right).
    \end{equation*}
    Then, there exists $\varepsilon_0 > 0$ such that for all $\varepsilon \le \varepsilon_0$ the function $\Gamma_\varepsilon : \R_c \to \R_c$ given by 
        \begin{equation*}
            \Gamma_\varepsilon(\x) = 
            \begin{cases} 
                1, & \text{ if } \|f(\x)- \hat{f}(\x)\|< \varepsilon, \\ 
                0, &  \text{ otherwise }
            \end{cases}  
        \end{equation*} 
    is not computable.
\end{Proposition}

\section*{Acknowledgements}
This work of H. Boche was supported in part by the German Federal Ministry of Education and Research (BMBF) within the national initiative on 6G Communication Systems through the Research Hub 6G-life under Grant 16KISK002.

This work of G. Kutyniok was supported in part by the Konrad Zuse School of Excellence in Reliable AI (DAAD), the Munich Center for Machine Learning (BMBF) as well as the German Research Foundation under Grants DFG-SPP-2298, KU 1446/31-1 and KU 1446/32-1. Furthermore, G. Kutyniok acknowledges support from LMUexcellent, funded by the Federal Ministry of Education and Research (BMBF) and the Free State of Bavaria under the Excellence Strategy of the Federal Government and the Länder as well as by the Hightech Agenda Bavaria

\bibliography{sn-bibliography}

\begin{appendices}

\section{(Semi-)decidability of real sets}
\label{app:semi}
In this section, we provide further background on the (semi-)decidability of subsets of real numbers. For more details, we refer to \cite{computableanalysisbook, Zhou96CompSets, Iljazovic2021CompMeticSpace, Brattka2003CompMetricSpace, Parker2006DecConcepts}.

First, note that feasible notions of computability exist beyond Borel-Turing and Banach-Mazur computability. A common approach is to relax the computability requirements on the input domain. The underlying idea is to separate the mapping from the input description leading to the following definition of computable function, which we call oracle computability to distinguish it from the previous notions.
\begin{Definition}[Oracle model]
    For $\x\in \R^d$, a sequence $(\q_k)_{k=1}^\infty \subset \Q^d$ such that
    \begin{equation*}
        \norm{\x-\q_k} \leq \frac{1}{2^k} \quad \text{ for all } n \in \N,
    \end{equation*}
    is called an \emph{oracle representation} of $\x$. A function $f: D \goto \R_c^m$, where $D \sube \R^d$, is \emph{oracle computable} if there exists an Oracle Turing machine $M$ such that for all $\x \in D$ and all oracle representations $(\vect{q}_k)_{k=1}^\infty$ of $\x$ the sequence $(M(\vect{q}_k))_{k=1}^\infty$ is a representation of $f(\x)$.
\end{Definition}

\begin{Remark}
    Note that, unlike Borel-Turing computability, the representing rational sequence $(\vect{q}_k)_{k=1}^\infty$ is not required to be computable. By the density of $\Q$, any real number has an oracle representation and, therefore, we can study computability on the whole real line. 
    Intuitively, one can think of the sequences $(\vect{q}_k)_{k=1}^\infty$ being provided to the Turing machine by an oracle tape; for an introduction on Oracle Turing machines see cite{computabilitybook}. This model is more general, but in typical practical applications, the presence of an oracle able to approximate any real number to arbitrary precision cannot be assumed. 
\end{Remark}
The differences in the computability notion (and the respective input domains) in comparison with Borel-Turing computability also directly transfer to the (semi-)decidability of sets: Only trivial subsets of $\R^d$ are oracle decidable -- whereas for Borel-Turing decidability the same statement holds for $\R^d_c$. 
\begin{Definition}
    A set $A \sube D \cap \R^d$ is
    \begin{itemize}
        \item \emph{oracle decidable} in $D$, if its indicator function $1_A: D \goto \R_c$ is oracle computable;
        \item \emph{oracle semi-decidable} in $D$, if there exists an oracle computable function $f: D' \goto \R_c$, $D' \sube D$, such that $A \sube D'$ and $f = 1_A|_{D'}$.
    \end{itemize}
\end{Definition}    
Intuitively, oracle decidability, as well as Borel-Turing decidability, is infeasible in general (except for the trivial cases) since there does not exist an algorithm that decides on arbitrary input $x \in \R$ (via representations) whether $x=0$, $x>0$ or $x<0$ -- the crucial input is the edge case zero \cite{Pour-El17Computability}. Hence, the best one can hope for is a notion of decidability `up to equality': Instead of relying on the characterization of (semi-)decidability via characteristic functions, one can consider the (continuous and computable) distance function $d_A: \R^d \to \R$ of $A \sube \R^d$ defined by 
\begin{equation*}
    d_A(\x) := dist(\x, A)= \inf_{\vect{a}\in A} \norma{\x-\vect{a}}.  
\end{equation*}
The distance function allows for a given $\x \in \R^d$ to compute how close $\x$ lies to $A$ although, in general, we cannot determine whether $\x \in A$ or $\x \notin A$. Therefore, a decidability notion based on the distance function does not lead to the existence of algorithms deciding membership for a given set even though closed sets are uniquely determined by their distance function. 

Nevertheless, for subsets of natural numbers, one can derive an interesting connection between classical decidability and the distance function \cite{Brattka2003CompMetricSpace, Zhou96CompSets}.
\begin{Proposition}
     A subset $A \sube \N$ is decidable (in the classical sense), if and only if $A$ considered as a subset of the real numbers induces a (oracle) computable distance function.    
\end{Proposition}
A similar statement also holds for semi-decidable sets on $\N$. To that end, we introduce a specific characterization of oracle semi-decidable sets \cite{computableanalysisbook}.
\begin{Theorem}
    Let $V\subset \R^d$. The following are equivalent:
    \begin{enumerate}[label={(\roman*)}]
        \item $V$ is oracle semi-decidable.
        \item $V$ is recursively enumerable open, i.e., there exists a Turing machine that can enumerate centers $\vect{c}_k\in \Q^d$ and radii $r_k \in \Q_{>0}$ of open balls such that
        \begin{equation}\label{eq:RecEnOpen}
            V = \bigcup_{k\in\N} B(\vect{c}_k,r_k),
        \end{equation}
        i.e., there exist computable rational sequences $(\vect{c}_k)_{k=1}^{\infty},(r_k)_{k=1}^{\infty}$ such that \eqref{eq:RecEnOpen} holds.
    \end{enumerate}    
\end{Theorem}
\begin{Remark}
    Any oracle semi-decidable set has a specific structure, in particular, it is necessarily open.   
\end{Remark}
The observed equivalence also carries over to Borel-Turing semi-decidability by adjusting the expression in \eqref{eq:RecEnOpen} to
    \begin{equation}
        V\cap\R_c = \bigcup_{n\in\N} B(\vect{c}_n,r_n) \cap \R_c.
    \end{equation} 
To highlight the differences note that one can show under certain assumptions that an interval $(a,b) \sube \R$ with non-computable endpoints $a,b \in \R$ is oracle semi-decidable and thus also Borel-Turing semi-decidable. Moreover, $[a,b]$ can be Borel-Turing semi-decidable as well for specific choices of non-computable $a,b$ (since $[a,b]\cap \R_c = (a,b)\cap \R_c$), whereas $[a,b]$ is not oracle semi-decidable as a closed set. 
\begin{Proposition}[\cite{Brattka2003CompMetricSpace, Zhou96CompSets}]
    A set $A \sube \N$ is semi-decidable (in the classical sense), if and only if $A$ considered as a subset of the real numbers is recursively enumerable open.   
\end{Proposition}
Finally, we want to summarize and highlight the conclusions based on the introduced statements. 
Due to the extended input domain, oracle (semi-)decidability is the stronger condition than Borel-Turing (semi-)decidability. However, in both frameworks decidability is not a practical notion on the real numbers due to the inability to decide equality. Although semi-decidability is less restrictive, it is still rather impractical since only open sets are amenable to semi-decidability, e.g., closed sets or sets that are neither open nor closed do not fit the framework. Nevertheless, (semi-)decidability on real domains based on the distance function recovers the classical theory of (semi-)decidability on natural numbers indicating that the introduced definitions are indeed the right ones. The difference between the two domains is that non-(semi-)decidability does not arise due to the inability to test equality on natural numbers and is therefore much scarcer in this setting. In comparison, non-semi-decidability on the real domain may not necessarily be related to the inability to test inequality, however, due to this shortcoming non-semi-decidability is likely to occur if no further assumptions are posed on the considered sets. 

Hence, one might urge for more suitable notions of (semi-)decidability on the real numbers, which circumvent equality comparisons. For instance, a reliable description of 'near-decidability' indicates whether an object is in a set or not up to some limited error. Simply relying on the distance function does not immediately entail the desired property. In contrast, recursive approximability related to a measure $\mu$ \cite{Parker2006DecConcepts} describes the following setting: Given a parameter $n \in \N$, there exists an algorithm that correctly decides $A \sube \R^d$ except on some set $B \sube \R^d$ with $\mu(B) < 2^{-n}$, in which case the algorithm still halts but with possibly incorrect output. Hence, there is a trade-off between correctness and guaranteed termination of the computation in finite time. Thus, the approach measures the possible error via $\mu$. Borel-Turing semi-decidability on the other hand ensures that an algorithm always provides the correct output once it finalizes the computation but it may not stop for certain inputs. In other words, it only indicates whether a point is near another point that is not in the considered set. However, this information cannot be used to deduce whether the considered point lies inside or outside the set.

Further pursuing notions related to recursive approximability is certainly valuable and might lead to further insights; in this work, we consider the extension of the classical (semi-)decidability definitions that lead to oracle/Borel-Turing (semi-)decidability. These definitions describe the existence of effective (semi-)decision programs, i.e., algorithms that necessarily compute correct outputs (or do not halt their computations). In this sense, the output of the algorithm can be unequivocally trusted. Moreover, the theory as well as the results in Subsection \ref{sec:decide} can be extended to spaces with less structure than $\R^d$, e.g., to computable metric spaces with $\R^d$ being a special case thereof \cite{Iljazovic2021CompMeticSpace, Brattka2003CompMetricSpace}.

\section{Proofs}\label{secA1}
\subsection{Proof of Theorem \ref{learning}}
The proof of Theorem \ref{learning} is based on two results we present next. The key component of the first lemma lies behind many non-computability results, such as in \cite{hansen} for the special case of inverse problems, but here we formulate a general version.
\begin{Lemma}
	\label{lem:computability}
	Let $\Theta$ be a nonempty set, $\Lambda$ a nonempty set of functions from $\Theta$ to $\R_c$, $\varepsilon > 0$, and $\Xi: \Theta \goto \mathcal{P}(\R_c^m)$, where $m \in \N$ and $\mathcal{P}$ denotes the power set. Assume there exist sequences $(\iota_k^1)_{k=1}^\infty, (\iota_k^2)_{k=1}^\infty \sube \Theta$ satisfying
	\begin{enumerate}[label={(\roman*)},topsep=5pt,itemsep=5pt]
		\item\label{converge} $\abs{f(\iota_k^1)-f(\iota_k^2)} \goto 0$ uniformly in $f \in \Lambda$. That is,
		\begin{equation*}
			\forall \delta > 0 ~\exists k_0 \in \N ~\forall k \ge k_0 ~\forall f \in \Lambda: \abs{f(\iota_k^1)-f(\iota_k^2)} < \delta;
		\end{equation*}
		\item\label{separate} for all $k \in \N$, $\text{dist}(\Xi(\iota_k^1), \Xi(\iota_k^2)) > \varepsilon$.
	\end{enumerate}
	Then, for all $n \in \N$ and all Banach-Mazur computable functions $\Gamma: \R_c^n \goto \R_c^m$ there exists $\iota \in \Theta$ such that for all $(f_1, \dots, f_n) \in \Lambda^n$:
	\begin{equation*}
		\text{dist}(\Gamma(f_1(\iota), \dots, f_n(\iota)), \Xi(\iota)) > \frac{\varepsilon}{3}.
	\end{equation*}
\end{Lemma}
\begin{proof}
	For contradiction assume that for some $n \in \N$ there exists a Banach-Mazur computable function $\Gamma: \R_c^n \goto \R_c^m$ such that for all $\iota \in \Theta$ there exists $(f_1, \dots, f_n) \in \Lambda^n$ with 
    \begin{equation}\label{eq:ContAss}
        \text{dist}(\Gamma(f_1(\iota), \dots, f_n(\iota)), \Xi(\iota)) \le \frac{\varepsilon}{3} .   
    \end{equation}
    Since $\Gamma$ is Banach-Mazur computable, it is continuous on $\R_c^n$ \citep{computableanalysisbook}, that is,
	\begin{equation*}
		\forall \eta > 0 ~\exists \delta > 0 ~\forall \x_1, \x_2 \in \R_c^n: \norma{\x_1 - \x_2} < \delta \Rightarrow \norma{\Gamma(\x_1) - \Gamma(\x_2)} < \eta.
	\end{equation*}
	Take $\eta = \frac{\varepsilon}{3}$. For the corresponding $\delta$ there exists by condition \ref{converge}. some $k \in \N$ such that for all $f \in \Lambda$ we have $\abs{f(\iota_k^1) - f(\iota_k^2)} < \frac{\delta}{n}$. This implies for all $(f_1, \dots, f_n) \in \Lambda^n$ that 
	\begin{align*}
		\norma{(f_1(\iota_k^1), \dots, f_n(\iota_k^1)) - (f_1(\iota_k^2), \dots, f_n(\iota_k^2))} ~&= \sqrt{\sum_{i=1}^{n} \left(f_i(\iota_k^1)-f_i(\iota_k^2)\right)^2} \\
        \le \sum_{i=1}^{n} \sqrt{ \left(f_i(\iota_k^1)-f_i(\iota_k^2)\right)^2} &= \sum_{i=1}^{n} \abs{f_i(\iota_k^1)-f_i(\iota_k^2)} < \delta,
	\end{align*}
	and therefore
	\begin{equation*}
		\norma{\Gamma(f_1(\iota_k^1), \dots, f_n(\iota_k^1)) - \Gamma(f_1(\iota_k^2), \dots, f_n(\iota_k^2))} < \frac{\varepsilon}{3}.
	\end{equation*}
	Together with \eqref{eq:ContAss} we get
	\begin{align*}
		\text{dist}\left(\Xi(\iota_k^1),\Xi(\iota_k^2)\right) \le ~&\text{dist}\left(\Gamma(f_1(\iota_k^1), \dots, f_n(\iota_k^1)), \Xi(\iota_k^1)\right) ~+ \\
		&\norma{\Gamma(f_1(\iota_k^1), \dots, f_n(\iota_k^1)) - \Gamma(f_1(\iota_k^2), \dots, f_n(\iota_k^2))} + \\
		&\text{dist}\left(\Gamma(f_1(\iota_k^2), \dots, f_n(\iota_k^2)), \Xi(\iota_k^2)\right) \\
		< ~&3\frac{\varepsilon}{3} = \varepsilon,
	\end{align*}
	which contradicts condition \ref{separate}.
\end{proof}

The following is a reformulation of Theorem 4.2 from \cite{topology}, stating that there exist functions representable by neural networks that are arbitrarily close in the supremum norm but can only be represented by networks with weights arbitrarily far apart. The norm $\norma{\cdot}_{\text{scaling}}$ on the (parameter) space of neural networks is used in the mentioned theorem because it provides a bound on the Lipschitz constant of neural network realizations $\text{Lip}(R_\sigma^D(\cdot))$, i.e., $\text{Lip}(R_\sigma^D(\Phi)) \le C \norma{\Phi}_{\text{scaling}}$ for some $C>0$ and a network $\Phi$, thus connecting the parameter space and the function space.

\begin{Definition}\label{def:scaling}
	For a neural network $\Phi = ((A_\ell, \vect{b}_\ell))_{\ell=1}^L$ set 
    \begin{equation*}
        \norma{\Phi}_{\text{scaling}} := \max_{1 \le \ell \le L} \norma{A_\ell}_{\max} = \max_{1 \le \ell \le L} \max_{i,j} |(A_\ell)_{i,j}|.   
    \end{equation*}
\end{Definition}

\begin{Lemma}[{\citep[Theorem 4.2]{topology}}]
	\label{lem:topology}
	Let $\sigma: \R \goto \R$ be Lipschitz continuous, but not affine linear. Let $S = (d, N_1, \dots, N_{L-1}, 1)$ be an architecture of depth $L \ge 2$ with $N_1 \ge 3$. Let $D \sube \R^d$ be bounded with a nonempty interior. 
	Then there exist sequences $(\Phi_k)_{k=1}^\infty, (\mu_k)_{k=1}^\infty \sube \mathcal{NN}(S)$ such that
	\begin{enumerate}[label={(\roman*)},topsep=5pt,itemsep=5pt]
		\item $\|R_\sigma^D(\Phi_k) - R_\sigma^D(\mu_k)\|_\infty \goto 0$,
		\item\label{diverg} for any $(\Phi_k^\prime)_{k=1}^\infty, (\mu_k^\prime)_{k=1}^\infty \sube \mathcal{NN}(S)$ with $R_\sigma^D(\Phi_k^\prime) = R_\sigma^D(\Phi_k)$ and $R_\sigma^D(\mu_k^\prime) = R_\sigma^D(\mu_k)$ for all $k \in \N$, it holds that $\norma{\Phi_k^\prime - \mu_k^\prime}_{\text{scaling}} \goto \infty$. 
	\end{enumerate}
\end{Lemma}
\begin{Remark}\label{rm:topology}
	It can be shown that the divergence in point \ref{diverg} is uniform in the following sense:
	\begin{align*}
		\forall \varepsilon > 0 ~\exists k_0 ~\forall k \ge k_0&\\
        \forall \Phi_k^\prime, \mu_k^\prime \in \mathcal{NN}(S)& ~\text{such that}~ R_{\sigma}^D(\Phi_k^\prime) = R_{\sigma}^D(\Phi_k), R_{\sigma}^D(\mu_k^\prime) = R_{\sigma}^D(\mu_k):\\
        &\norma{\Phi_k^\prime - \mu_k^\prime}_{\text{scaling}} > \varepsilon.
	\end{align*}
	To see this, assume $R_\sigma^D(\mu_k) \equiv 0$, and for contradiction let there be a subsequence $(\Phi_{k_\ell}^\prime)_{\ell=1}^\infty \sube \mathcal{NN}(S)$ with $R_\sigma^D(\Phi_{k_\ell}^\prime) = R_\sigma^D(\Phi_{k_\ell})$ and $\norma{\Phi_{k_\ell}^\prime}_{\text{scaling}} \le \varepsilon$ for some $\varepsilon > 0$. Then for some $C > 0$:
	\begin{equation*}
		\text{Lip}(R_\sigma^D(\Phi_{k_\ell})) = \text{Lip}(R_\sigma^D(\Phi_{k_\ell}^\prime)) \le C \norma{\Phi_{k_\ell}^\prime}_{\text{scaling}} \le C \varepsilon,
	\end{equation*}
	which contradicts $\text{Lip}\left(R_\sigma^D(\Phi_{k_\ell})\right) \goto \infty$ in condition \ref{diverg}.
 
	From the proof in \cite{topology} it can also be seen that for a computable $\sigma$ at least one such pair of these sequences of neural networks lies in $\mathcal{NN}_c(S)$.
\end{Remark}

\begin{proof}[Proof of Theorem \ref{learning}]
	Let $\Theta = \left\{R_\sigma^D(\Phi) \separ \Phi \in \mathcal{NN}_c(S) \right\}$. For $i \in \{1, \dots, d\}$ and $\x \in D$ denote by $f_{\x}^i: \Theta \goto \R_c$ the constant operator 
    \begin{equation*}
        f_{\x}^i(g) = x_i    
    \end{equation*}
    and by $f_{(\x)}: \Theta \goto \R_c$ the operator 
    \begin{equation*}
        f_{(\x)}(g) = g(\x).     
    \end{equation*}
    Let $\Lambda = \left\{f_{\x}^i \separ \x \in D, i \in \{1, \dots, d\} \right\} \cup \left\{f_{(\x)} \separ \x \in D\right\}$ and define $\Xi: \Theta \goto \mathcal{P}(\R_c^{N(S)})$  by
    \begin{equation*}
        \Xi(g) = \left\{\Phi \separ R^D_\sigma(\Phi) = g \right\}.    
    \end{equation*}
	By Lemma \ref{lem:topology} there exists a pair of sequences $(g_k)_{k=1}^\infty, (h_k)_{k=1}^\infty \sube \Theta$ such that $\norma{g_k - h_k}_\infty \goto 0$. Therefore also $\abs{f_{(\x)}(g_k) - f_{(\x)}(h_k)} \goto 0$ uniformly in $\x \in D$. The same trivially holds for all $f_{\x}^i$, therefore condition \ref{converge}. of Lemma \ref{lem:computability} is satisfied.

    By Remark \ref{rm:topology}, the sequences diverge uniformly in the scaling norm and therefore also in the Euclidean norm, meaning $\text{dist}(\Xi(g_k), \Xi(h_k)) \goto \infty$ and, in particular, for any $\varepsilon > 0$ there exists $k_0$ such that $\text{dist}(\Xi(g_k), \Xi(h_k)) > 3\varepsilon$  for $k \ge k_0$. Hence, condition \ref{separate}. of Lemma \ref{lem:computability} holds with $3\varepsilon$.
	
	Together, by Lemma \ref{lem:computability} for all $n \in \N$ and all Banach-Mazur computable functions $\Gamma: (\R_c^d\times \R_c)^n \goto \R_c^{N(S)}$ there exists $g \in \Theta$, such that for all $\left(f_1, \dots, f_{n(d+1)}\right) \in \Lambda^{n(d+1)}$ we have 
	\begin{equation*}
		\text{dist}\left(\Gamma(f_1(g), \dots, f_{n(d+1)}(g)), \Xi(g)\right) > \varepsilon.
	\end{equation*}
    However, by construction of $\Lambda$ and $\Xi$, this entails that there exists $\Phi \in \Xi^{-1}(g)$, i.e., $\Phi  \in \mathcal{NN}_c(S)$, such that for all $\x_1, \dots, \x_n \in D$ and all $\Phi^\prime \in \mathcal{NN}_c(S)$ with $R_\sigma^D(\Phi^\prime) = R_\sigma^D(\Phi) =g$ we have
	\begin{equation*}
		\norm[2]{\Gamma(\mathcal{X}) - \Phi^\prime} > \varepsilon.
	\end{equation*}  
\end{proof}

\subsection{Proof of Theorem \ref{thm:posLearning} and \ref{thm:posLearning2}}
The key component of the proofs in this section relies on enumerating rational neural networks, i.e., networks with rational parameters, and subsequently controlling the error induced thereby. 
\begin{proof}[Proof of Theorem \ref{thm:posLearning}]
    First, enumerate the countable set of rational neural networks $\{\Phi_1, \Phi_2, \dots \}$ of the given architecture, in particular, we can associate $\Q^{N(S)}$ with $\{\Phi_1, \Phi_2, \dots \}$. For all $\hat{\Phi}, \Phi^\ast \in \mathcal{NN}_{c}(S)$, $\hat{\Phi}$ is a computable function so that 
    \begin{equation*}
        g_{\hat{\Phi}, \Phi^\ast}(\x) := \left| R_\sigma^{\R_c^d}(\hat{\Phi})(\x) - R_\sigma^{\R_c^d}(\Phi^\ast)(\x)\right|     
    \end{equation*}
    is computable. Assume the training data $\mathcal{X}= \{(x_1,y_1), \dots, (x_n,y_n)\}$ was generated by a neural network $\Phi$. Next, we construct an algorithm that correctly recognizes whether $g_{\Phi_k,\Phi}(\x_i) < \varepsilon$ for all $i = 1, \dots, n$: Compute $g_{\Phi_k,\Phi}(\x_i)$ with precision (at least) $\tfrac{1}{2}\varepsilon$ and subsequently check whether the magnitude of the obtained (rational) number is smaller than $\tfrac{1}{2}\varepsilon$. If so, the algorithm returns $\Phi_k$, if not, it continues by increasing $k \in \N$.  

    Moreover, by the density of rational networks, there exists a rational network $\Phi_{k_0}$ such that for all $i = 1, \dots, n$: $g_{\Phi_{k_0},\Phi}(\x_i) < \tfrac{1}{2}\varepsilon$. Hence, the algorithm terminates not later than the $k_0$-th iteration, returning a correct answer. This characterizes a computable function $\Gamma$ satisfying the claim.
\end{proof}
Extending the proof by incorporating the additional conditions yields Theorem \ref{thm:posLearning2}. 
\begin{proof}[Proof of Theorem \ref{thm:posLearning2}]
    Fix some architecture $S$. First, observe that for arbitrary $\Phi, \hat{\Phi} \in \mathcal{NN}_{c}(S)$ and $\mathcal{X} \in \mathcal{D}_{\Phi,\R_c^d}^n$  the following holds:
    \begin{align*}
        | \Phi(\x) - \hat{\Phi}(\x) | &\leq  | \Phi(\x) -  \Phi(\hat{\x})| + |\Phi(\hat{\x}) - \hat{\Phi}(\hat{\x})| + |\hat{\Phi}(\hat{\x}) - \hat{\Phi}(\x) | \\
        &\leq \| \x - \hat{\x}\| (\text{Lip}(R_\sigma^D(\Phi)) + \text{Lip}(R_\sigma^D(\hat{\Phi}))) + |\Phi(\hat{\x}) - \hat{\Phi}(\hat{\x})|,
    \end{align*}
    where $\x \in D$ and $\hat{\x} = \text{argmin}_{\{\x_i : (\x_i, y_i) \in \mathcal{X}} \| \x - \x_i \|$. Therefore, using Definition \ref{def:scaling} we get for arbitrary $r>0$ and $(\x,y) \in \mathcal{X}^{\Phi}_r$ 
    \begin{equation*}
        | y - \hat{\Phi}(\x) |  \leq r C(\sigma, S) (\norma{\Phi}_{\text{scaling}} + \|\hat{\Phi}\|_{\text{scaling}}) + |\Phi(\hat{\x}) - \hat{\Phi}(\hat{\x})|, 
    \end{equation*}
    where $C(\sigma, S)>0$ is a computable constant depending on the architecture and the activation function. Hence, applying Theorem \ref{thm:posLearning} shows that for all $\tilde{\varepsilon} > 0$ and $n\in\N$ there exists a computable function $\Gamma_{\tilde{\varepsilon} }: (\R_c^d\times \R_c)^n \goto \R_c^{N(S)}$ such that for all $\Phi \in \mathcal{NN}_{c}(S)$ and  $\mathcal{X} \in \mathcal{D}_{\Phi,\R_c^d}^n$ we have 
    \begin{equation*}
        \big| R_\sigma^{\R_c^d}\left(\Gamma_{\tilde{\varepsilon} }(\mathcal{X})\right)(\x) - y \big| < r C(\sigma, S) (\norma{\Phi}_{\text{scaling}} + \|\Gamma_{\tilde{\varepsilon} }(\mathcal{X})\|_{\text{scaling}}) + \tilde{\varepsilon} \quad \text{ for } (\x,y) \in \mathcal{X}^{\Phi}_r.
    \end{equation*}
    Restricting to input networks $\Phi$ with $\norma{\Phi}_{\text{scaling}} \leq A_{\max}$ and setting 
    \begin{equation*}
        r^\ast = \frac{\varepsilon - \tilde{\varepsilon}}{C(\sigma, S) (A_{\max} + \|\Gamma_{\tilde{\varepsilon} }(\mathcal{X})\|_{\text{scaling}})} \quad \text{ for some } \varepsilon > \tilde{\varepsilon}
    \end{equation*}
    gives
    \begin{equation*}
        \left| R_\sigma^{\R_c^d}\left(\Gamma_{\tilde{\varepsilon} }(\mathcal{X})\right)(\x) - y \right| < \varepsilon  \quad \text{ for } (\x,y) \in \mathcal{X}^{\Phi}_{r^\ast}.
    \end{equation*}
    Finally, for given $\varepsilon>0$ and $n\in \N$, set $\Gamma = \Gamma_{\tfrac{1}{2} \varepsilon}$ and define $\Psi: \R_c^{N(S)} \to \R_c$ by
    \begin{equation*}
        \Psi(\Phi) =  \frac{\varepsilon }{2C(\sigma, S) (A_{\max} + \|\Phi\|_{\text{scaling}})}.
    \end{equation*}
    Observing that $\Gamma$ satisfies \eqref{eq:GenGuarantee} and $\Psi$ is a computable function (since $C(\sigma, S)$ and $\|\cdot\|_{\text{scaling}}$ are computable, and we may assume without loss of generality that $\varepsilon$ and $A_{\max}$ are computable) gives the claim.      
\end{proof}

\subsection{Prof of Theorem \ref{quantizedlearning} and \ref{quantizedlearningV2}}
An enumeration argument similar to the ones in the previous proofs implies Theorem \ref{quantizedlearning}. In particular, the idea is to encode the target network as a single datapoint, which can be done recursively for integer vectors representing neural networks with integer parameters.
\begin{proof}[Proof of Theorem \ref{quantizedlearning}]
    Given an architecture $S = (d, N_1, \dots, N_{L-1}, 1)$, $\mathcal{NN}_{\Z}(S)$ can be associated with $\Z^{N(S)}$, which in turn can be recursivelly encoded into $\Z^d$ by a recursivelly invertible function $g: \Z^{N(S)} \goto \Z^d $ (see for instance \cite{computabilitybook} for details). Then, taking $\Gamma(\mathcal{X}) = g^{-1}(\x_1)$ with $\mathcal{X}= \{(\x_1, y_1) \dots, (\x_n,y_n)\}$, a single datapoint of the form $\big(g(\Phi), R_\sigma^{\Z^d}(\Phi)(g(\Phi)), 0, \dots\big) \in (\Z^d\times \Z)^n$ can be used to reconstruct any neural network $\Phi \in \mathcal{NN}_{\Z}(S)$. Here we utilize the fact, that we can choose the dataset for each network specifically.
\end{proof}
By taking the training data more explicitly into account Theorem \ref{quantizedlearningV2} follows.

\begin{proof}[Proof of Theorem \ref{quantizedlearningV2}]
    Given a dataset $\mathcal{X}=\{(\x_i,y_i)\}_{i=1}^n \in \mathcal{D}_{\Phi,\Z^d}^n$, enumerate the countable set of all neural networks $\{\Phi_1, \Phi_2, \dots \}$ with a given architecture $S$, in particular, we can associate $\Z^{N(S)}$ with $\{\Phi_1, \Phi_2, \dots \}$. For increasing $k \in \N$, check whether for all $i = 1, \dots, n$: $R_\sigma^{\Z^d}(\Phi_k)(\x_i) = y_i$. If so, return $\Phi_k$, if not, continue.

    If the data was generated using a neural network $\Phi=\Phi_{k_0}$, then the algorithm terminates at the latest in the $k_0$-th iteration, returning a correct answer. This characterizes a (partially) recursive function $\Gamma$ satisfying the theorem.
\end{proof}
\end{appendices}
\end{document}